%% file: rlc_main.tex
\documentclass[10pt]{article} % For LaTeX2e

 \usepackage[accepted]{rlj} % Should be uncommented for the camera-ready
\usepackage{amssymb}            % Defines common symbols like \mathbb R
\usepackage{amsthm}
\usepackage{mathtools}          % Extends amsmath, providing common math tools
\usepackage{mathrsfs}           % Enables \mathscr, which can work in cases that \mathcal does not
\usepackage{graphicx}           % For including images
\usepackage{subcaption}         % Allows for the use of subfigures and subcaptions
\usepackage{wrapfig}
\usepackage[space]{grffile}     % For spaces in image names
\usepackage{url}                % For displaying URLs
\usepackage{lipsum}             % For placeholder text
\usepackage[ruled,vlined]{algorithm2e}
\usepackage[colorinlistoftodos,textsize=tiny]{todonotes}

\newtheorem{theorem}{Theorem}

\newtheorem{fact}{Fact}
\theoremstyle{definition}
\newtheorem{definition}{Definition}

\title{Rethinking the Foundations for Continual \\
Reinforcement Learning}

\setrunningtitle{Rethinking the Foundations for Continual Reinforcement Learning}

\author{Esraa Elelimy\textsuperscript{1,2}, David Szepesvari\textsuperscript{1,2}, Martha White \textsuperscript{1,2,3}, Michael Bowling\textsuperscript{1,2,3}}

\emails{\{elelimy,dszepesv,whitem,mbowling\}@ualberta.ca}

\affiliations{
$^{1}$\textbf{Department of Computing Science, University of Alberta, Canada}\\
$^{2}$\textbf{Alberta Machine Intelligence Institute (Amii)}\\
$^{3}$\textbf{CIFAR AI Chair}\\

\par % If including additional comments like below, use \par to add some whitespace. 
}

\contribution{
     We identify four foundational principles and practices that shape and constrain our thinking about RL.
    We argue that these foundations, shaped by the traditional framing of RL, are antithetical to the purported goals of continual reinforcement learning and may be holding us back from making progress toward continual learning.
    }
    {
    Most of our arguments are in alignment with the constraints that arise under the big world hypothesis \citep{javed2024the}.
    Previous work by~\citet{abel2024dogmasreinforcementlearning} has discussed three dogmas that shape most reinforcement learning research. The second dogma overlaps with the second foundation that we argue against as part of the foundations of traditional reinforcement learning. In this work, we identify three additional problematic foundations in the traditional RL framing and, for completeness, we reiterate some of the arguments against this second dogma as well.
    }

\contribution{
    We present a new formalism that replaces two of the foundations with the history process as a mathematical formalism and deviation regret as an evaluation metric.
    }
    {
    The history process foundation is built on earlier work by~\citet{BowlingEtAl2023-settling} and~\citet{hutter2000theory}, and the deviation regret is an extension to earlier work by~\citet{MorrillEtAl21-hindsight}. The earlier work on deviation regret by~\citet{MorrillEtAl21-hindsight} focused on settings where the history process can be repeated, such as in extensive-form games. In this work, we extend the notion of deviation regret to the continual learning setting and provide some theoretical analysis on deviation regret estimation in the continual learning setting.
    }
\contribution{We present experimental results suggesting that the current RL algorithms fail to learn continually and that our proposed measure of evaluation can evaluate those failures.}
{~\citet{plataniosjelly} showed similar results for agents failing to learn continually, which aligns with our experimental findings. We extend those results to show the utility of deviation regret as an evaluation measure when agents fail to learn.}

 % Your keywords

\summary{In the traditional view of reinforcement learning, the agent's goal is to find an optimal policy that maximizes its expected sum of rewards. Once the agent finds this policy, the learning ends. This view contrasts with \emph{continual reinforcement learning}, where learning does not end, and agents are expected to continually learn and adapt indefinitely. Despite the clear distinction between these two paradigms of learning, much of the progress in continual reinforcement learning has been shaped by foundations rooted in the traditional view of reinforcement learning. In this paper, we first examine whether the foundations of traditional reinforcement learning are suitable for the continual reinforcement learning paradigm. We identify four key pillars of the traditional reinforcement learning foundations that are antithetical to the goals of continual learning: the Markov decision process formalism, the focus on atemporal artifacts, the expected sum of rewards as an evaluation metric, and episodic benchmark environments that embrace the other three foundations. We then propose a new formalism that sheds the first and the third foundations and replaces them with the history process as a mathematical formalism and a new definition of deviation regret, adapted for continual learning, as an evaluation metric. Finally, we discuss possible approaches to shed the other two foundations.

}
\input{macros}

\begin{document}

\makeCover  % Create the cover page
\maketitle  % Make the title section

\begin{abstract}
In the traditional view of reinforcement learning, the agent's goal is to find an optimal policy that maximizes its expected sum of rewards. Once the agent finds this policy, the learning ends. This view contrasts with \emph{continual reinforcement learning}, where learning does not end, and agents are expected to continually learn and adapt indefinitely. Despite the clear distinction between these two paradigms of learning, much of the progress in continual reinforcement learning has been shaped by foundations rooted in the traditional view of reinforcement learning. In this paper, we first examine whether the foundations of traditional reinforcement learning are suitable for the continual reinforcement learning paradigm. We identify four key pillars of the traditional reinforcement learning foundations that are antithetical to the goals of continual learning: the Markov decision process formalism, the focus on atemporal artifacts, the expected sum of rewards as an evaluation metric, and episodic benchmark environments that embrace the other three foundations. We then propose a new formalism that sheds the first and the third foundations and replaces them with the history process as a mathematical formalism and a new definition of deviation regret, adapted for continual learning, as an evaluation metric. Finally, we discuss possible approaches to shed the other two foundations.
\end{abstract}

\section{Introduction}
``Consider a Markov decision process defined by the tuple ...'' starts many background sections of reinforcement learning (RL) papers.  The Markov Decision Process (MDP) formalism, among other foundational concepts, has long shaped how we think about agents, algorithms, and evaluation in RL. 
However, these foundational concepts stemmed from a classical framing of the RL problem: \emph{an agent's goal is to find an optimal policy that maximizes its expected sum of rewards}. Once this policy is found, the learning ends -- the agent no longer needs to adapt because the policy is, by definition, optimal. That traditional view influenced many of the foundations and standard practices in the field. For example, a direct consequence of the view that learning ends with finding an optimal solution is to have a separate training phase with the goal of finding that optimal solution and then have a deployment phase where no more learning is happening. Another consequence is the emphasis on the artifacts that the training process produces and overlooking the behavior of the agent during learning.

There are many decision-making problems where the traditional framing of RL is a shortcoming. For example, agents acting in a world that is much bigger and more complex than themselves,  such that they cannot perceive or represent its true underlying state, will neither be able to represent the value of the states they find themselves in nor find an optimal policy~\citep{javed2024the}. Such agents can only rely on approximate solutions that continually adapt to perceived changes in their environment and improve as they accumulate more knowledge by interacting with the world. In these types of decision-making problems, learning is no longer about finding an optimal solution but about continual and never-ending adaption. This class of decision-making problems, where continual adaption is necessary, is called continual reinforcement learning~\citep{AbelEtAl23-definition}.

While the traditional and the continual learning views of RL share some similarities --- they both tackle the problem of learning by interacting with the world --- they have a crucial difference: framing learning as a means to find optimal artifacts versus learning as an indefinite process of adaption. Given this core difference, it is essential to reflect on whether the foundations that has stemmed from the traditional view still hold and are helpful when addressing the continual learning problem. Or could these traditional foundations hold us back from thinking most usefully about the problem?  

In the first part of this paper, we identify four traditional foundational principles and practices that shape and constrain our thinking about RL.
We argue that these foundations, shaped by the traditional framing of RL, are antithetical to the purported goals of continual reinforcement learning and may be holding us back from making progress toward continual learning. Moreover, these foundations are self-reinforcing: each depends upon and holds up the others, such that when attempting to replace one, the others constrain you to keep it.

In the second part of the paper, we propose a new formalism that sheds two of these foundations, and we discuss possible alternatives for replacing the remaining two foundations.

\section{Four Foundations of Traditional RL}
\label{sec:foundations}
Most reinforcement learning research, along with recent progress in continual reinforcement learning, make the following assumptions, implicity or explicitly:
\begin{enumerate}
    \item \textbf{Formalism:} The appropriate mathematical formalism is the \emph{Markov decision process}.
    \item \textbf{Objective:} The goal of RL algorithms is to produce \emph{atemporal artifacts} (such as an optimal policy or value function).
    \item \textbf{Evaluation:} The ideal measure of evaluation is the \emph{expected sum of rewards}.
    \item \textbf{Benchmarking:} Most benchmarks for comparing RL algorithms are \emph{episodic environments}.
\end{enumerate}

These assumptions are the pillars of the traditional RL foundations and remain pervasive within modern RL research. Celebrated results such as DQN reaching human-level performance in Atari \citep{Mnih2015HumanlevelCT}, AlphaGo \citep{silvermastering}, GT-Sophy \citep{wurman2022outracing}, balloons in the stratosphere \citep{bellemare2020autonomous}, and DeepStack beating professional poker players \citep{moravvcik2017deepstack} all embody these foundations. They undergo a separate training phase in episodic environments respecting common MDP assumptions such as ergodicity and communicating dynamics. This training process generates atemporal artifacts (policy or value functions) that are considered optimal or near-optimal. These artifacts are then evaluated according to their expected sum of rewards in an evaluation phase where no more learning occurs. 

While these foundations were behind most of the advancement of traditional RL research, do they give us an appropriate structure to pursue continual reinforcement learning? Continual reinforcement learning does not have a consensus definition~\citep {ring1994continual,AbelEtAl23-definition}. However, its very name implies that learning should continue. We now discuss that this conclusion alone is enough to create cracks in those four foundations, and we will briefly summarize the alternatives that could replace those traditional foundations.

\textbf{Foundation One: MDPs as a Mathematical Formalism.}
This foundation is concerned with the assumptions on the environment that typically accompany the MDP formalism. We often make ergodicity assumptions, such as the MDP being unichain or communicating, which imply some characteristics of the environment. For example, we may implicitly assume every state is reachable from every other state or that the state distribution converges to some stationary distribution. Furthermore, we usually presume some properties of the MDP, such as finite state and action spaces or compact spaces with continuity assumptions. There are some problems where these assumptions hold and the MDP formalism works well. In grid-world environments, for instance, an agent can revisit any state as often as needed. In Go, repeatedly playing the game in episodes guarantees a form of ergodicity since it allows the agent to repeatedly visit previous game states by replaying the same sequence of moves. However, an important observation is that these are also examples where continual learning is unnecessary.

In contrast, the need for continual learning arises in settings with unpredictable non-stationarity in the environment
~\citep{khetarpal2022towards} or those that align with the \emph{big world hypothesis}~\citep{javed2024the}. The \emph{big world hypothesis} suggests that even if the real world is stationary, its complexity is much richer than the representational capacity of any agent in it. Hence, the world will appear unpredictably non-stationary. When acting in a much more complex world or when there are constraints on the computational resources of the agent, continual learning is needed \citep{kumar2023continual,dong2022simple}, even if the underlying world is stationary \citep{trackingsutton}. 
In these settings, the predictable stationarity of MDPs is invalid. Moreover, real-world settings do not allow one to reset the world into repeatable episodes or revisit states previously visited. You, the reader, can never revisit the state before you read these words. This inability to revisit states renders ergodicity assumptions unrealistic for real-world settings. 

\textit{The Alternative: History Processes as a Mathematical Formalism.} Beyond the agent-environment interface, this formalism has few assumptions about the process since the \emph{big world hypothesis} does not allow the agent to assume a priori structure or regularity about the environment. We expand on this foundation formally in Section~\ref{sec:history_process}. 

\textbf{Foundation Two: Focus on Atemporal Artifacts.}
Artifacts refer to any atemporal representation of an agent's learned knowledge, such as policies, value functions, options, or features. We often give considerable concern to the notion of optimal value functions and optimal policies. The assumption that learning should produce those fixed representations leads us to think of algorithms having a ``training'' period wherein they aim to converge to optimal artifacts and follow that with a ``testing'' phase to evaluate the generated artifacts. These artifacts exist for some problems, such as the grid world and chess examples, but they do not exist for problems that require continual learning.

Environments of interest to continual learning rarely admit fixed optimal artifacts. The assumption that an agent can converge to an optimal policy or a value function contradicts the very need for continual adaption since such an atemporal artifact would be the end of learning rather than requiring its continuation. For example, consider an agent with computational constraints that cannot fully represent the values of all possible states in its environment. For that agent, even if a fixed optimal value function theoretically exists, it cannot represent it, compute it, or store it. Instead,
such an agent must rely on an approximation of this value function that evolves over time, deciding which information to retain and which to discard. In this context, the most useful value representation is continually adapting and time-dependent, not atemporal. As a result, a focus on fixed atemporal artifacts should be replaced with a focus on the continual adaption of the agent's behavior. 
This foundation is also notably critiqued as \emph{Dogma Two} by \citet{abel2024dogmasreinforcementlearning}.

\textit{The Alternative: Focus on Behaviour.}
The goal of RL algorithms is to produce behavior in response to experience. In the continual learning setting, there is no difference between training and testing.  All the past experience is training, and all future experience is testing.  The focal point is how an agent behaves in response to its experience.

\textbf{Foundation Three: Expected Sum of Rewards as an Evaluation Measure.}
In episodic environments, this is the episodic return, and we desire that during training, we see the episodic return approach the return of the optimal policy. Episodes allow drawing i.i.d. samples of this return for any stationary policy,  which is how evaluation is usually performed during the testing phase. Hence, maximizing the episodic return during training often leads to better performance during the testing phase. 

A salient feature of real-world settings that require continual learning is the inability to reset the world or revisit previous states, i.e., the MDP may not be communicating, as discussed in Foundation One. A ramification of this feature is that it is not even possible to estimate an \emph{expected sum of rewards} as it would require
the environment to be repeatedly reset to something akin to an initial state so that the agent can reliably try different actions in the same states to achieve the optimal performance criteria. One might think the average reward criterion in a continuing environment is a solution to this criticism.  However, without the communicating assumption, a high average reward may be more a property of how fortunate the agent is to end up in a particular communicating class of states with a high average reward.
For continual learning settings, we need a measure of evaluation that does not depend on having such a repeatability assumption. 

\textit{The Alternative: Deviation Regret as an Evaluation Measure.}
We propose deviation regret as an evaluation measure for continual learning agents. Deviation regret was proposed as the evaluation measure for defining hindsight rationality, originally introduced in the context of strategic games~\citep{MorrillEtAl21-hindsight}. We further develop this concept for continual reinforcement learning.  The essence is that agents should be evaluated on the ``situations'' they find themselves in, not against some optimal, unrealizable sequence of actions. We formulate deviation regret for continual learning in Section~\ref{hsr}.

\textbf{Foundation Four: Episodic Benchmarks.}
Common environments, such as classic control tasks and the Arcade Learning Environment (ALE, \citet{bellemare2013arcade}), are episodic and, therefore, are communicating MDPs. Other naturally continuing environments, such as Mujoco \citep{todorov2012mujoco} and Minecraft, are often truncated during training, converting them into episodic tasks. A few examples of continuing, never-ending environments, such as Jelly Bean World \citep{plataniosjelly} exist but have not been widely adopted. 

Most of these traditional \emph{benchmarks} are problematic when considering the goal of continual learning. They reinforce the idea that environments can always be thought of as ergodic and episodic and exhibit an optimal policy, which is the assumed goal of traditional RL training.

\textit{The Alternative: Benchmark Environments Without a Clear Markov State or Episode Reset.}  We will not expand on this much beyond recognizing that it as an issue.  In summary, we should not expect to see continual learning algorithms differentiate themselves in environments where continual learning is unnecessary. Additionally, more work is needed to design environments where continual learning is needed.  To make progress, we should have benchmarks that align with the big world hypothesis. Ideally, we should test our agents in the complex, big, real world, but this is impractical for algorithmic development and scientific repeatability. An alternative is to constrain our agents' representational capacity and use more modest-sized environments such that the constraints simulate the big world hypothesis and allow for the development of agents that can cope in such continual learning settings.

\textbf{Final Remarks on the Traditional Foundations.}
The four foundations we discussed are self-reinforcing.  Just presuming the goal of artifacts immediately suggests the MDP formalism to support the existence of an optimal policy and necessary assumptions to ensure it can be learned, with benchmark environments that fit these assumptions.  Similarly, our common benchmark environments have a clear notion of optimal policy, making the focus be on algorithms that produce such an artifact.  It is no simple task to tear down any one of these foundations when the others demand its reinstatement.  Hence, our proposed alternatives seek to replace all four of these foundations.

\section{History Processes as a Mathematical Formalism}
\label{sec:history_process}
For the new formalism to support the goals of continual RL, we need to place as few constraints on the environment as possible. Ideally, constraints would be limited to the interface between the environment and the agent (e.g., actions, observations, rewards) but not on the properties of the environment or its dynamics (e.g., Markovianity, ergodicity).  One might consider this as an impossible approach as there needs to be some structure or repeatability in the environment to make learning possible.  We will resolve this by making post hoc statements as is common with bandit algorithms, e.g., this agent performs nearly as well as the single best arm in hindsight. Such statements can be made for stationary bandits (with assumptions on the environment) and for adversarial bandits (where limited assumptions are made).

We base the environment definition on the formalism introduced by \citet{BowlingEtAl2023-settling}, which had a similar aim to approach environments and goals as generally as possible. We deviate slightly from this formalism by assuming that the agent acts first, as in the work by \citet{AbelEtAl23-definition}. Formally, we assume a finite action space, $\Action$, and a finite observation space $\Obs$.  
We can then define the space of finite-length histories as $\Hist \equiv \bigcup_{n=0}^\infty (\Action \times \Obs)^n$, which is the set of all possible sequences of observation-action pairs that can result from the agent-environment interaction. We then define the environment as follows:
\begin{definition}
An environment $\env$ is a function from finite-length histories and actions to a distribution over observations,
$\env : \Hist \times \Action \rightarrow \Delta(\Obs)$.
\end{definition}
Finally, we assume that the agent's goal is a preference relation over histories that satisfies the reward hypothesis axioms~\citep{BowlingEtAl2023-settling}, including temporal $\gamma$-indifference.  Hence, it can be represented as a reward function mapping from actions and observations to a real-valued number: $\reward : \Action \times \Obs \rightarrow {\mathbb R}$, where the agent's goal is to maximize the expected $\gamma$-discounted sum of rewards $\reward(a_t, o_t)$, summed over the transitions in its history. 
Since the domain of this function is the finite set of actions and observations, the range of this reward function is bounded.

We continue to follow \citet{BowlingEtAl2023-settling} and define an agent as follows:
\begin{definition}
An agent $\agent$ is a function from finite-length histories to a distribution over actions, 
$\agent : \Hist \rightarrow \Delta(\Action)$.
\end{definition}
We will focus on agents that can be decomposed into a \emph{representation of state} and a system that learns to select policies over this representation. Formally, let $\State$ be a finite set, which we will call states, and let $S : \Hist \rightarrow \State$ be some fixed partition of the histories such that $S(h) \in \State$ is the agent's representation of the state for history $h$. Using this state representation, we can specify a notion of a policy, $\pi : \State \rightarrow \Delta(\Action)$, as a mapping from a state to a distribution over actions, with $\Pi$ being some fixed set of such mappings.
Finally, we define an agent's learning rule as follows:
\begin{definition}
The agent's learning rule $\sigma$ is a function from finite-length histories to a distribution over policies,
$\sigma : \Hist \rightarrow \Delta(\Pi)$.
\end{definition}

To illustrate how these definitions interact, consider the history at time $t$, $h_t \equiv \left<a_1, o_1, \ldots, a_t, o_t\right>$. Given that history, the agent takes an action $a_{t+1} \sim \pi_{t}\left(S(h_t)\right)$ where $\pi_t = \sigma(h_t)$.  The environment then generates an observation $o_{t+1} \sim e(h_t, a_t)$ creating the new history $h_{t+1}$.

\textbf{Remarks.}
The use of state here should not be confused with the requirements on the state as used in an MDP, such as Markovianity. It is not intended to restrict the dynamics of the environment, it is the agent's own representation of the history.  One may require $S$ to be defined in the form of a state update function, $u : \State \times \Action \times \Obs \rightarrow \State$, that defines how states evolve in a recurrent fashion with each each transition from a starting state $s_0$ as in \citet{MorrillEtAl22-partially}.

This kind of decomposition of the agent into a fixed state representation and an adapting policy is explicitly seen in \citet{MorrillEtAl22-partially} and \citet{dong2022simple}, and implicitly in \citet{AbelEtAl23-definition}.  In the latter, they introduce the notion of an \emph{agent basis}: $\AgentBasis \subset \Agent$, and a learning rule that maps histories to an element of the agent basis.  We are essentially choosing $\Pi$ as our agent basis $\AgentBasis$, and we allow the learning rule $\sigma$ to map to a distribution over the agent basis, i.e., over the policy set $\Pi$.  As with~\citet{AbelEtAl23-definition}, we will examine the agent's learning through its learning rule $\sigma$ that is adapting the choice of policy $\pi_t$ from its experience, $h_t$.  

\section{Deviation Regret as an Evaluation Measure}\label{hsr}
Given the history process formalism, we now turn our attention to a measure of evaluation. 

\subsection{Agents as Creators of Worlds}
Given an environment $e$ and a finite-length history $h$, we can construct a new environment, $\env_h(h', a) \equiv \env(h \cdot h', a)$, 
which defines the set of distributions over observations that arise from actions taken after history $h$.  This matches our mathematical formalism for an environment.  Thus, as an agent acts in its environment instantiating a sequence of histories $h_1, h_2, \ldots, h_n$, it can be seen as also instantiating a sequence of \emph{worlds}, each world is itself an environment, $\env_{h_1}, \env_{h_2}, \ldots, \env_{h_n}$.  {\bf An effective learning agent should be well-adapted to the worlds that it finds itself in.}  We will attempt to instantiate this notion using {\em deviation regret}, extending the notion of hindsight rationality from ~\citet{MorrillEtAl21-hindsight}, which focused on the setting where there is a repeatability of the history process, to continual learning where there is no repeatability.

We now define a deviation $\phi$ as a function that systematically applies modifications to the agent's policy. Formally, a \emph{deviation} is defined as $\dev : \Pi \rightarrow \Pi$, where $\Pi$ is the set of all possible policies. For example, a deviation might change the action taken at a singular state, or if the agent's policy is a parametrized function, it might apply a systematic perturbation to the parameters of the policy, generating a new deviation policy.
As we discussed in section~\ref{sec:history_process}, the agent's learning rule $\sigma$ generates the agent's policy at each time step given the history up to that time step, i.e., $\pi_t = \sigma(h_t)$. To study an agent under a deviation, we apply the deviation $\dev$ to the agent's policy in each timestep, producing the deviation policy $\phi(\pi_t)$.
Hence, we can further define a function that composes the agent's learning rule with the deviation function: $\phi(\sigma): \Hist \rightarrow \Delta(\Pi)$.

Deviation regret focuses on the notion of a systematic deviation. 
For any particular deviation, we care about the agent's {\em regret} for not applying the deviation, and we sum this regret over opportunities to apply this deviation. In our case, the sequence of opportunities is the sequence of worlds instantiated by the agent's own interaction with the environment. This gives us a deviation regret for deviation $\dev$ in environment $\env$ by agent $\lambda$,
\begin{align}\label{eq:dev_regret_def}
\!\!\underbrace{\regret_T(\phi, \agent, \env)}_{\mbox{\small deviation regret}} \!\!=
\frac{1}{T}\sum_{t=1}^{T} \bigg(
\underbrace{\E\left[\sum_{i=t}^{t+H-1} \gamma^{(i-t)} R_i \bigg| \phi(\sigma), H_{t-1}\right]}_{\mbox{\small deviation return}} - 
\underbrace{\E\left[\sum_{i=t}^{t+H-1} \gamma^{(i-t)} R_i \bigg| \sigma, H_{t-1}\right]}_{\mbox{\small agent return}}
\bigg)
\end{align}
where $H$ is an evaluation horizon chosen so $\gamma^H$ is sufficiently small, and $H_t$ is the history (and corresponding world) experienced by the agent in timestep $t$.
An important note is that we discount rewards at time $i$ with $(i-t)$, since this new world starts at time $t$, with all previously accumulated rewards $r_1, \ldots, r_{t-1}$ shared by both the deviation return and the agent return (so they cancel in the difference).  The purpose of discounting in this way is to treat each world equally rather than treating later worlds as discounted by the time since the beginning of the interaction.

As is common with regret notions, we are interested in whether $\regret_T(\dev, \agent, \env) \rightarrow 0$, i.e., the deviation regret is approaching zero almost surely or in expectation for any environment.  And if this holds for all deviations $\dev \in \Dev$, we say that the agent is minimizes deviation regret with respect to the set of deviations $\Dev$.
What do we choose for the set $\Dev$?  This question has interesting answers in the repeated extensive-form game setting \citep{MorrillEtAl21-hindsight, MorrillEtAl21-efficient}, but as one concrete example, we might consider $\Dev$ to be the class of {\em external deviations}.  An external deviation is a constant function, i.e., $\dev_\pi(\cdot) \equiv \pi$.  So we can consider $\Dev_{\text{ext}} = \left\{ \dev_\pi \right\}_{\pi\in\Pi}$.  In this case, deviation regret is comparing the agent's expected return to the expected return of a fixed policy averaged over the worlds experienced by the agent.  With no additional assumptions on the environment, this would necessitate an agent that continually learns. Furthermore, as an evaluation measure, deviation regret focuses on the agent's behavior in response to its experience, shifting the focus away from artifacts.

\subsection{Deviation-Regret Estimation}
We now show that an agent can estimate the deviation regret given its stream of experience. The definition of deviation regret in Eq.~\ref{eq:dev_regret_def} consists of two components: the agent return and the deviation return. The rewards along the trajectory of the agent directly estimate the agent return. The deviation return may seem unknowable as it requires a counterfactual estimate of the return under an alternative sequence of policies. However, just as with adversarial bandits, we can estimate the counterfactual return of having applied a deviation as long as the agent's support for policies is always closed under the deviation function, so that one can compute an importance sampling ratio $\frac{\Pr(a_i | \dev(\pi_i))}{\Pr(a_i | \pi_i)}$ and construct an unbiased estimator of the deviation return with bounded variance. This can be achieved by a sufficiently random learning rule.
A precise algorithm for the deviation regret estimate is given in Algorithm~\ref{alg:estimate}. While the presented algorithm uses ordinary importance sampling, practical implementations may use other importance sampling variants or variance reduction methods. 
\begin{algorithm}
{\small
\caption{Estimating the Deviation Regret $\hat{\regret}_T(\phi, \lambda, e)$}
\label{alg:estimate}
\DontPrintSemicolon % Removes default semicolons
\KwIn{Deviation $\phi$, agent $\lambda$, horizon $H$, trajectory $\{(h_{t-1}, a_t, r_t)\}_{t=1}^{T}$}
\KwOut{Estimated deviation regret $\hat{\regret}_T(\phi, \lambda, e)$}
Initialize $\hat{G}_T \gets 0$, $\hat{G}'_T \gets 0$ \;
\For{$t = 1$ to $T$}{
    Compute $G_t \gets \sum_{i=t}^{\min(T, t+H-1)} \gamma^{i-t} r_i$ \;
    Compute importance weight $W_t \gets \prod_{i=t}^{\min(T, t+H-1)} \frac{\phi(\pi_i)(a_i | h_{i-1})}{\pi_i(a_i | h_{i-1})}$ \;
    Update $\hat{G}_T \gets \hat{G}_T + G_t$, $\hat{G}'_T \gets \hat{G}'_T + W_t G_t$ \;
}
Compute $\hat{\regret}_T(\phi, \lambda, e) \gets \frac{1}{T} (\hat{G}_T - \hat{G}'_T)$ \;
\Return $\hat{\regret}_T(\phi, \lambda, e)$ \;
}
\end{algorithm}

Now that we have an estimator for the deviation regret, we are interested in understanding the quality of that estimator. For the case where we have a finite horizon $H$, Theorem \ref{thm:h-step_estimate} states that if the agent's policy is sufficiently random, the deviation regret estimator in Algorithm~\ref{alg:estimate} is consistent. i.e, as the agent's experience grows, the agent's estimate of the deviation regret gets arbitrarily close to the true deviation regret, with probability approaching $1$. Additionally, we extended those results to the case where we might have an infinite horizon, $H = \infty$. Theorem ~\ref{eq:consistency} states that we can have a consistent estimator for the case $H = \infty$ and $\gamma < 1$. 
We provide proofs for the finite horizon case in Appendix~\ref{sec:theorem_1_proof} and for the infinite case in~\ref{sec:inf_consistent_proof}. While the statements here are asymptotic, the appendix contains a finite sample bound for the $H$-step deviation return estimator.

\begin{theorem}[Estimating the $H$-step Deviation Regret]
\label{thm:h-step_estimate}
The estimator we defined above, $\hat \regret_T (\phi, \agent, \env)$, is a consistent estimator of deviation regret $\regret_T(\phi, \agent, \env)$ for all environments $\env$, deviations $\phi$, $\gamma \in [0, 1]$, and agents $\agent$ that take every action with probability at least $c > 0$ in every timestep.
More precisely, for all $\varepsilon > 0$, 
\begin{equation} \label{eq:consistency}
    \lim_{T\to\infty} \PP \big(
        |\regret_T (\phi, \agent, \env) - \hat \regret_T (\phi, \agent, \env) | \leq \varepsilon \big) = 1,
\end{equation}
where the probability is taken over the random behaviour of the agent acting in the environment.
\end{theorem}

\begin{theorem}
\label{thm:inf_estimate}
There is a consistent estimator for the case where $\gamma < 1$ and $H=\infty$.
\end{theorem}

\subsection{Illustrative Experiments}
In this section, we present an illustrative experiment that demonstrates the utility of deviation regret as a measure of evaluation for continual learning agents. Traditional evaluation measures, such as average rewards or episodic returns, can indicate when the agent's learning degrades or when the agent stops learning; a drop in the average rewards would indicate that. However, they fall short in characterizing those failures. i.e., they do not answer questions such as, \emph{Was there a better policy that would have been more effective, but the agent failed to find it?}. Deviation regret, on the other hand, addresses such questions, providing more insight into the agent's behaviour. 

The goal of our experimental analysis is threefold: First, to show that current RL algorithms, developed around traditional RL foundations, often fail in continual learning settings. Second, showing that deviation regret can identify those failures. i.e., there is a positive deviation regret when those failures happen. Third, showing that there exists a policy, representable by the agent, that would have avoided those failures. 

\begin{wrapfigure}{r}{0.4\textwidth}
    \centering
    \includegraphics[width=0.4\columnwidth]{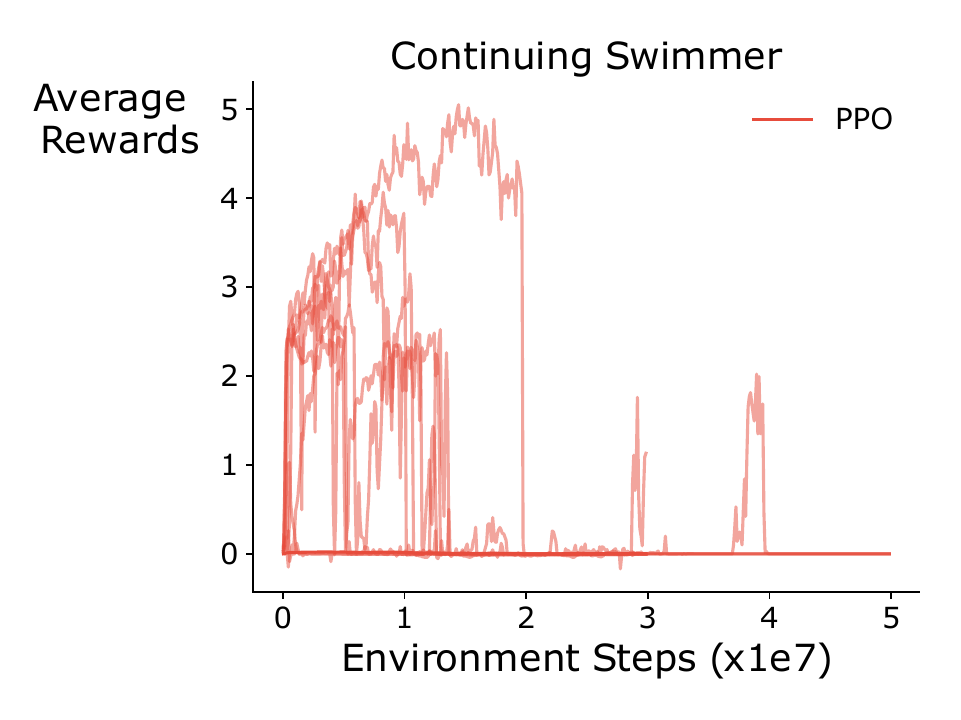}
    \caption{PPO fails to continually learn in the Continuing Swimmer environment. We show that across different random seeds, the agent loses the ability to learn after around $20$M steps.}
    \label{fig:swimmer_runs}
\end{wrapfigure}

\textbf{Current algorithms fail to continually learn.}
To study agents' behaviors when there is no repeatability or resets in the environment, we modified the Swimmer environment from Mujoco~\citep{todorov2012mujoco} and turned it into a continuing task. 
We then trained a PPO~\citep{schulman2017proximal} agent in this Continuing Swimmer environment for $50$ million steps and repeated the experiment using $10$ different seeds. We show the hyperparameters of the PPO agent in Table~\ref{tab:ppo_hypers}, which are based on the commonly used hyperparameters for PPO with Mujoco environments~\citep{shengyi2022the37implementation}.
Figure~\ref{fig:swimmer_runs} shows the results of this experiment, where across all seeds, agents started learning for some time, and then they all failed. While some seeds managed to learn for longer than others, after $20$ million steps of interaction with the environment, all agents had already failed to continue learning.

\textbf{There exists a deviation policy when agents fail to learn.}
For the second and third goals of this experiment, we wanted to show that deviation regret identifies the agent's failure and to show that there exists a policy representable by the agent that would have avoided such failure. We can achieve these two goals simultaneously by choosing a deviation set that is representable by the agent. Then, if there is a positive regret had the agent used any deviation policy from this deviation set, we can easily conclude that those two goals are achieved. 

\begin{figure}[htb]
    \centering
    \includegraphics[width=\linewidth]{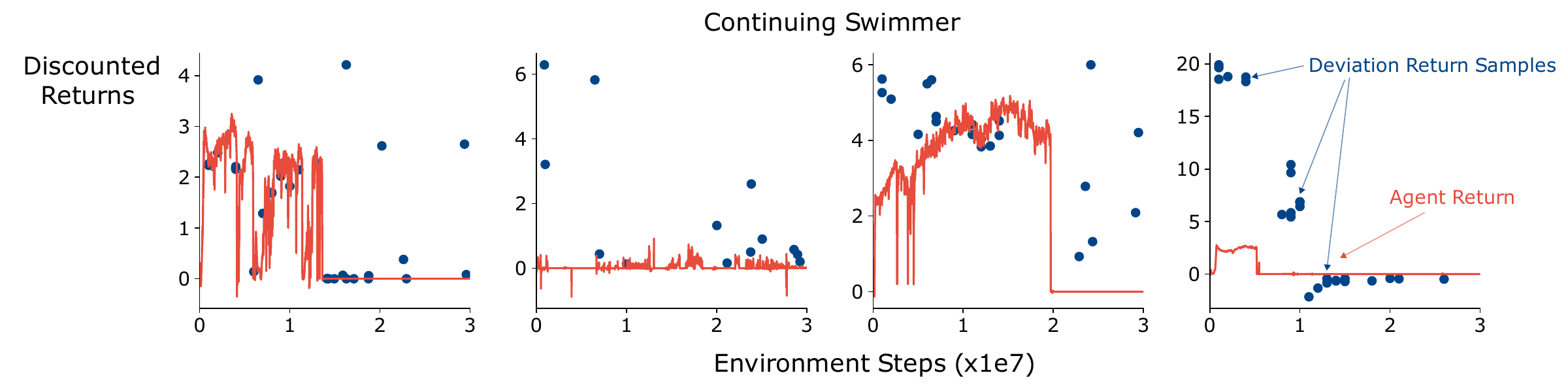}
    \caption{We show that when agents fail to learn continually, they have a positive deviation regret. Each red line represents one of the seeds for the previously shown agent in Figure \ref{fig:swimmer_runs}; here, we are showing the agent's discounted H-step return. i.e, the second term in Eq.\ref{eq:dev_regret_def}. The blue dots represent samples of the deviation return, the first term in Eq.\ref{eq:dev_regret_def}, from the best policy in the deviation set. When the blue dots are above the red lines, that means following the deviation policy would have resulted in a better return than the one achieved by the agent. Note that the best deviation policy is representable by the agent, given how we constructed the deviation set.}
    \label{fig:dev_returns}
\end{figure}

To ensure that the agent has the capacity to represent the policies in the deviation set,
we constructed the deviation set, $\Dev$, from different checkpoints of the neural network weights. Each checkpoint contains a snapshot of the network parameters at different points during learning. We then estimated the deviation regret of the agents had they used any of these deviation policies. Finally, we selected the best deviation policy for each agent and sampled estimates for its return starting from various history points. i.e, sampled estimates of the deviation return in Eq.~\ref{eq:dev_regret_def}.
Figure~\ref{fig:dev_returns} shows the discounted H-step return for the agents, and the sampled deviation returns starting from different histories. When the deviation return sample is higher than the agent's return, then there is a positive deviation regret. We can see that when agents fail, the return from the deviation policy is almost always higher than the agent's return, meaning that if the agent had used this deviation policy, it wouldn't have failed. 

\section{Discussion}~\label{app:objection}
We now address a number of objections
that can be raised against this notion of deviation regret and the history process formalism.
\paragraph{Deviations give an alternative and unknowable sequence of worlds.} A potential challenge is that systematically applying a deviation would change the distribution of worlds encountered by the agent, which is an unknowable counterfactual.  A critical distinction in the choice of deviation regret is that we are not doing {\em policy regret}~\citep{AroraEtAl12-adaptive}, where the environment within which the deviation's return is evaluated is affected by the applied deviation.  However, we also are not making any ``oblivious adversary'' assumption that the distribution of worlds is not impacted by the agent's actions, i.e., we have an adaptive adversary.  Typically, this setting is met with responses such as external regret does not admit any natural interpretation when the adversary is adaptive \citep{AroraEtAl18-policy}.  The interpretation though is clear, it reflects how much the agent would prefer to have applied the deviation to its policy under the sequence of worlds it actually found itself in; whether that is a natural interpretation seems at least debatable.  Note that a similar choice is made in off-policy reinforcement learning, where the excursion setting considers the target policy's effect on future states and rewards from the distribution of states visited by the behavior policy rather than correcting the distribution to fit the target policy's distribution if it were to be followed~\citep{SuttonEtAl16-emphatic}.
Furthermore, there are settings where vanishing external regret implies vanishing policy regret~\citep{AroraEtAl18-policy}, which are exactly recovered in games where this notion was first explored.  Most importantly, though, this approach does not need the unknowable counterfactual.

This distinction between policy regret and deviation regret can be observed with an environment that is constructed as a two-state MDP.  The actions are {\sc stay} or {\sc switch}, which deterministically cause their respective transition.  The reward for {\sc switch} is always $-10$ while the reward for {\sc stay} is $+1$ in state 1, and $+2$ in state 2 (the initial state), and $\gamma = \frac{1}{2}$.  Policy regret would compare any agent to the policy that always chooses {\sc stay} never leaving the initial state and its discounted return is $4$.  However, an agent that followed this policy does not guarantee no policy regret (or deviation regret for many deviations), as the adversary could just as easily set the reward for {\sc switch} and {\sc stay} in state 1 high enough for it to suffer linear regret.  Now consider an agent that avoids this outcome via doing some degree of exploration.  At some point it will end up in state 1, and once in state 1 the best policy to maximize future discounted return is to {\sc stay} forever for a return of $+2$.  Policy regret would consider this a poor outcome.  However, does it really make sense to look back in time and compare the agent's future behavior from state 1 to what would have been possible if it had never ever taken the {\sc switch} action to leave state 2?  Once in state 1, the comparison should be to what can be done to maximize discounted return in the world it finds itself in.  That is the heart of deviation regret.  Finally, note that $\gamma$ (or the evaluation horizon $H$) is playing a significant role in the notion of deviation regret.\footnote{This is in contrast to the ``futility of discounting in continuing problems'' from \citet[p. 254]{SuttonBarto18-rlbook}, where the choice of discount factor is shown not to affect the agent's objective.  The difference from our treatment is their appeal to a {\em stationary distribution}, which requires an ergodicity assumption on the environment we explicitly avoid.  Maybe discounting in a continuing problem such as our history process is not``futile'' after all?}  If $\gamma$ was large enough, the optimal policy would, in fact be to {\sc switch} back to state 2 and {\sc stay} forever.  And in such a case, policy regret and deviation regret would coincide.

\paragraph{Deviation regret does not order agents.}
A desirable property of an evaluation criteria is that you can use it to order agents.  We might desire to say that if 
$
\max_{\dev\in\Dev} \regret(\dev, \agent, \env) < 
\max_{\dev\in\Dev} \regret(\dev, \agent', \env),
$
then $\agent$ is preferred to $\agent'$ in environment $\env$.  However, this doesn't mean what it appears to mean.  Agent $\agent$ likely observes a different sequence of histories, and so a different distribution of worlds, compared to $\agent'$, and as a result, it is not at all clear what it would mean to compare the deviation regret over those worlds.  Notice that the above notion of policy regret allows for this kind of comparison since the comparator in the regret term does not depend on the agent at all.  This is a fair objection.  It does not seem possible to construct an intuitive total ordering using these criteria (however, note that it does seem possible to make an intuitive partial ordering).  Deviation regret is best used to judge if an agent is adapting effectively and to do so without making assumptions on the environment (e.g., assuming the environment is a finite ergodic MDP, where effective adaptation would necessarily converge to the MDP's optimal policy). 
Empirical leaderboards and benchmarks may still need to resort to expected discounted return on an environment.  However, that approach has its own weaknesses, particularly if we do not require ergodicity assumptions.

We can observe these different weaknesses in a simple environment where an agent must choose between {\sc left} or {\sc right} as its first action.  Suppose {\sc left} deterministically results in the agent playing repeated games of rock-paper-scissors against an opponent that always chooses {\sc rock}, so that there is a simple learning problem.  While {\sc right} results in the agent playing repeated games of Go against a strong but imperfect opponent, so there is a challenging learning problem.  The agent is completely uninformed in this decision.  However, considering simple expected discounted return on this environment, an agent that defaults to choosing its first action as its first decision will most definitely outperform any agent that orders its actions differently or chooses randomly.  This is true even if this alternative agent is extremely capable at learning, and manages to eventually learn to win the majority of its games of Go.  Deviation regret, instead evaluates agents by whether they are effectively adapted, relative to some set of deviations, to the worlds in which they find themselves --- whether that be a simple to learn rock-paper-scissors setting or a challenging game of Go.  Since the above two agents don't see the same distribution of such worlds, it makes little sense to order the agents by this criteria.  However, it makes equally little sense to order them by how they make one completely uninformed decision, which would dominate any expected return assessment.

\paragraph{Deviation Regret encourages agents to reach a place where no learning is possible.}
We can avoid reaching places where no learning is possible by sublinearly increasing the evaluation horizon H. So even if the agent reaches a ``no learning place'', it will incur deviation regret that encourages it to change its policy and eventually get out of it. However, if no deviation policy incurs a deviation regret along the ``no learning path'', then the agent is doing the best it can given the world it found itself in.

\section{Conclusion and Future Work}
In this paper, we described four foundations of traditional RL that are antithetical to the goals of continual reinforcement learning.  Further, we presented the underpinnings of an alternative set of foundations that better conceptualize the challenges faced within continual learning.  
More excitingly, these foundations seem to suggest a new approach to agent and algorithm design.  This will also entail the development of suitable benchmark environments that embrace these alternative foundations. 
%We hope this work spurs on all of these lines of research.

%\subsubsection*{Broader Impact Statement}
%\label{sec:broaderImpact}
%In this optional section, RLJ/RLC encourages authors to discuss possible repercussions of their work, notably any potential negative impact that a user of this research should be aware of. 

%%%%%%%%%%%%%%%%%%%%%%%%%%%%%%%%%%%%%%%%%%%%%%%%%%%%%%%%%%%%%%%%
%% Appendices
%%%%%%%%%%%%%%%%%%%%%%%%%%%%%%%%%%%%%%%%%%%%%%%%%%%%%%%%%%%%%%%%
% \appendix
\section*{Appendix}
\label{sec:theorem_1_proof}
%\section{The first appendix}
%\label{sec:appendix1}
%This is an example of an appendix. 

%\noindent \textbf{Note:} Appendices appear before the references and are viewed as part of the ``main text'' and are subject to the 8--12 page limit, are peer reviewed, and can contain content central to the claims of the paper. 

% \section{Proofs}
% \label{sec:appendix_proofs}
Here, we give a proof of Theorem~\ref{thm:h-step_estimate} and provide an outline for the proof of Theorem~\ref{thm:inf_estimate}, establishing the consistency of the estimators. The complete proof of Theorem~\ref{thm:inf_estimate} can be found in the Supplementary Materials~(\ref{sec:inf_consistent_proof}).
The proofs primarily rely on the concentration of the return estimates around the true returns, shown by identifying a relevant martingale and applying the Azuma-Hoeffding inequality.
The above statement is true for idealized return estimates that can access future data, while in practice such data is not available.
Therefore we provide an error analysis of this difference and establish that as we see more data, the contribution of this error diminishes to $0$.
Following these two ideas, we can prove the consistency of the $H$-step deviation regret estimator.
The infinite discounted return case requires some additional care, which we discuss later.
In the remainder of this section we introduce the relevant notation, state some basic results used, and then give the proofs.

Any behaviour $f: \Hist \to \Delta(\Action)$ in an environment $e$ induces a distribution over trajectories.
Let $R_t$ be the random reward at timestep $t$.
We define $G_t = \sum_{i=t}^{t+H-1} \gamma^{i-t} R_i$, the random $H$-step discounted return from timestep $t$. Furthermore, $H_t$ denotes the random histories induced by the agent in the environment.
Then, the deviation regret may be re-expresed as
\begin{equation}
\regret_T(\phi, \agent, \env) =
    \frac{1}{T}\sum_{t=1}^{T} 
    \E \left[G_t | \phi(\sigma), H_{t-1} \right] - 
    \E \left[G_t | \sigma, H_{t-1} \right].
\end{equation}

Recall, when estimating the returns in Algorithm~\ref{alg:estimate}, we only had access to data up to timestep $T$, that is, the return estimates for the last $H$ steps are truncated. To capture these, define
\[
    G_t^{[T]} = \sum_{i=t}^{\min(T, t+H-1)} \gamma^{i-t} R_i
    \quad \text{ and } \quad 
    W_t^{[T]} = \prod_{i=t}^{\min(T, t+H-1)} \frac{\phi(\pi_i)(A_i | H_{i-1})}{\pi_i(A_i | H_{i-1})},
\]
where $A_t$ is the random action of the agent in timestep $t$, $G_t^{[T]}$ is the truncated agent return estimate, and ${G_t'}^{[T]} = W_t^{[T]} G_t^{[T]}$ the truncated deviation return estimate. Recall, $\pi_i$ is the policy used by the agent in timestep $i$.
The non-truncated (idealized) agent return estimate is captured by $G_t$, and, defining
$W_t = \prod_{i=t}^{t+H-1} \frac{\phi(\pi_i)(a_i | h_{i-1})}{\pi_i(a_i | h_{i-1})}$, $G_t' = W_t G_t$ is the non-truncated (idealized) deviation return estimate. Note, an apostrophe (or prime) denotes a quantity relevant to the deviation. With this,
\[
    \hat \regret_T (\phi, \agent, \env)
    = \frac1T \sum_{t=1}^T {G_t'}^{[T]} - G_t^{[T]}.
\]
We will argue through the idealized estimator
\(
    \hat \regret_T^* (\phi, \agent, \env)
    = \frac1T \sum_{t=1}^T G_t' - G_t.
\)

Finally, let $r^* = \max_{a \in \Action, o \in \Obs} |R(a, o)|$, which exists since $\Action$ and $\Obs$ are finite. Then
\begin{equation} \label{eq:g_t_bound}
    |G_t| \leq \sum_{i=t}^{t+H-1} \gamma^{i-t} |R_i| \leq Hr^*,
\end{equation}
and for an agent that takes every action with probability at least $c$ in every timestep
\begin{equation} \label{eq:g_t_prime_bound}
    |G_t'|
    = | W_t | \cdot | G_t|
    \leq \left |\prod_{i=t}^{t+H-1} \frac{\phi(\pi_i)(a_i | h_{i-1})}{\pi_i(a_i | h_{i-1})} \right| Hr^*
    \leq c^{-H}Hr^*.
\end{equation}

\begin{proof}[Proof of Theorem~\ref{thm:h-step_estimate}]
Fix any $\phi, \env$, and $\agent$ as in the theorem statement. Let $H_t$ be as above. For brevity, we will simply denote the deviation regret by $\regret_T$, and the estimates as $\hat \regret_T$ and $\hat \regret_T^*$.
We decompose the estimator error as
\begin{equation} \label{eq:thm1:err_decomp}
    |\regret_T -  \hat\regret_T|
    = |\regret_T - \hat \regret_T^* + \hat \regret_T^* - \hat\regret_T|
    \leq   |\regret_T - \hat \regret_T^* |
        +  | \hat \regret_T^* - \hat \regret_T |,
\end{equation}
where we used the triangle inequality. The first term is the estimation error for the idealized estimate, the second term is the difference between the idealized and the truncated estimates. We bound each of these separately.

\paragraph{Idealized Estimator Error.}
Consider
\[
    \Delta_T = T(\regret_T -  \hat\regret_T^*) = \sum_{t=1}^T (G_t' - \E[G_t | \phi(\sigma), H_{t-1}]) + (G_t - \E[G_t | \sigma, H_{t-1}]).
\]
Let $\delta_t = (G_t' - \E[G_t | \phi(\sigma), H_{t-1}]) + (G_t - \E[G_t | \sigma, H_{t-1}])$, the terms in the sum. We will show that $\delta_t$ is a martingale difference sequence (MDS), hence $\Delta_T$ is a martingale. Towards this, we need $\E[\delta_t | \sigma, H_{t-1}] = 0$ and $|\delta_t|$  bounded. We have
\begin{align*}
    \E[\delta_t | \sigma, H_{t-1}]
    &= \E \Big[ (G_t' - \E[G_t | \phi(\sigma), H_{t-1}]) + (G_t - \E[G_t | \sigma, H_{t-1}]) \,\Big|\, \sigma, H_{t-1} \Big] \\
    &= \E [ G_t' | \sigma, H_{t-1} ]  - \E \Big[ \E[G_t | \phi(\sigma), H_{t-1}] \,\Big|\, \sigma, H_{t-1} \Big]
        + \E [ G_t | \sigma, H_{t-1} ]  - \E [G_t | \sigma, H_{t-1}]
\end{align*}
by linearity of expectation. The last two terms cancel, and now we show the first two do as well. Since $\E[G_t | \phi(\sigma), H_{t-1}]$ is the expected deviation return (a constant), the second term itself is the same constant. In the first term $ G_t' = W_t G_t$, where $W_t$ is the importance sampling that corrects from the agents' behaviour to the deviation's behaviour. That is, it is well known that
\[
    \E [ G_t' | \sigma, H_{t-1} ]
    = \E [ W_t G_t | \sigma, H_{t-1} ]
    = \E [ G_t | \phi(\sigma), H_{t-1} ],
\]
and we see that the first two terms also cancel. We conclude $\E[\delta_t | \sigma, H_{t-1}] = 0$.

We bound $|\delta_t|$ by noting that each term in it is bounded. More precisely, we use Eq.~\ref{eq:g_t_bound}, Eq.~\ref{eq:g_t_prime_bound}, the triangle inequality, and that for any $X < c$, we also have $\E[X|Y] < c$ for all $Y$.
\[
    |\delta_t| 
    \leq | G_t' | + | \E[G_t | \phi(\sigma), H_{t-1}])| + |G_t| + |\E[G_t | \sigma, H_{t-1}]) |
    \leq c^{-H}Hr^* + 3Hr^*.
\]
Therefore, $\Delta_T$ is a martingale with the increments, $\delta_t$, bounded by $(c^{-H} + 3)Hr^*$. Note that $\Delta_0 = 0$. By the Azuma-Hoeffding inequality, we conclude that for any $\epsilon_1 > 0$ and $T \geq 1$
\begin{equation*}
    \PP(|\Delta_T| \geq \epsilon_1) \leq 2\exp \left( \frac{- \epsilon_1^2}{2T \big((c^{-H} + 3)Hr^*\big)^2} \right).
\end{equation*}
Since $|\Delta_T| \geq \epsilon_1$ is equivalent to $|\regret_T -  \hat\regret_T^*| \geq \epsilon_1/T$, letting $\epsilon_2 = \epsilon_1 / T$ we can restate the above as
\begin{equation} \label{eq:thm1:hp_bound}
    \PP(|\regret_T -  \hat\regret_T^*| \geq \epsilon_2)
    \leq 2\exp \left( \frac{-(\epsilon_2 T)^2}{2T \big((c^{-H} + 3)Hr^*\big)^2} \right)
    = 2\exp \left( \frac{-\epsilon_2^2 T}{2 \big((c^{-H} + 3)Hr^*\big)^2} \right).
\end{equation}
We see that the idealized estimate gets close to the true deviation at an exponential rate. This completes the bound on the first error term.

\paragraph{Error due to truncating the estimate.}
By definition and the triangle inequality
\[
    T | \hat \regret_T^* - \hat \regret_T |
    = \left| \sum_{t=1}^T ({G_t'}^{[T]} - G_t^{[T]}) - (G_t' - G_t) \right|
    \leq \sum_{t=1}^T | {G_t'}^{[T]} - G_t'| + |G_t^{[T]} - G_t |.
\]
Both  $\sum_{t=1}^T | {G_t'}^{[T]} - G_t'|$ and $\sum_{t=1}^T |G_t^{[T]} - G_t |$ are bounded by $r^*H^2$. It is only the last $H$ terms that are truncated, therefore all other terms in the sum are 0. Each of the non-zero terms are no more than $H$-step returns of rewards no more than $r^*$. Plugging this into the inequality we developed so far, we find
\begin{equation} \label{eq:thm1:trunc_err_bound}
    | \hat \regret_T^* - \hat \regret_T | \leq \frac{2r^* H} {T (1-\gamma)}.
\end{equation}
As this is a uniform bound over all realizations of the random estimates, for any $\epsilon_3 > 0$, choosing $T \geq \frac{2r^*H}{\epsilon_3 (1-\gamma)}$, we have $| \hat \regret_T^* - \hat \regret_T | \leq \epsilon_3$.

\paragraph{Combining the Error Estimates.}
Using the error decomposition in Eq.~\ref{eq:thm1:err_decomp}, as well as the bounds Eq.~\ref{eq:thm1:hp_bound} and Eq.~\ref{eq:thm1:trunc_err_bound} developed for each error component, we can conclude that for any $\epsilon > 0$, $T \geq \frac{2r^*H}{(\epsilon/2) (1-\gamma)}$,
\[
    \PP(|\regret_T -  \hat\regret_T| \leq \epsilon)
    \geq 1 - 2\exp \left( \frac{-(\epsilon/2)^2 T}{2(c^{-H} + 3)Hr^*} \right).
\]
Letting $T \to \infty$, we conclude with the result we set out to prove,
\[
    \lim_{T\to\infty} \PP (
        |\regret_T  - \hat \regret_T | \leq \varepsilon ) = 1.
    \qedhere
\]
\end{proof}

%\section*{Appendix}
% No label, since this can't be referenced meaningfully with \ref{}.
%This format should only be used if there is a single appendix (unlike in this document).

\subsubsection*{Acknowledgments}
\label{sec:ack}
The authors would like to thank a number of colleagues whose insights refined this work, including John Martin, Dustin
Morrill, and David Sychrovsky. This work was supported by Amii, the Canada CIFAR
AI Chairs program, NSERC, and the Google DeepMind Chair in Artificial Intelligence. We would also like to thank the Digital Research Alliance of Canada for the
compute resources.
%Use unnumbered third level headings for the acknowledgments. All acknowledgments, including those to funding agencies, go at the end of the paper. Only add this information once your submission is accepted and deanonymized. The acknowledgments do not count towards the 8--12 page limit.

%%%%%%%%%%%%%%%%%%%%%%%%%%%%%%%%%%%%%%%%%%%%%%%%%%%%%%%%%%%%%%%%
%% NOTE: THIS MARKS THE END OF THE "MAIN TEXT"
%%%%%%%%%%%%%%%%%%%%%%%%%%%%%%%%%%%%%%%%%%%%%%%%%%%%%%%%%%%%%%%%

%%%%%%%%%%%%%%%%%%%%%%%%%%%%%%%%%%%%%%%%%%%%%%%%%%%%%%%%%%%%%%%%
%% Bibliography
%%%%%%%%%%%%%%%%%%%%%%%%%%%%%%%%%%%%%%%%%%%%%%%%%%%%%%%%%%%%%%%%
\bibliography{refs}
\bibliographystyle{rlj}

%%%%%%%%%%%%%%%%%%%%%%%%%%%%%%%%%%%%%%%%%%%%%%%%%%%%%%%%%%%%%%%%
% AUTHOR: If your paper has no supplementary materials, you may 
%         comment out the line below, which creates the title for
%         the supplementary materials.
%%%%%%%%%%%%%%%%%%%%%%%%%%%%%%%%%%%%%%%%%%%%%%%%%%%%%%%%%%%%%%%%
\beginSupplementaryMaterials
\section{There is a consistent estimator for the infinite discounted deviation regret}
\label{sec:inf_consistent_proof}
The proof for the infinite discounted deviation is analogous to the $H$-step return case, with only two changes.
\begin{enumerate}
    \item The errors in the truncated estimates are larger, albeit still bounded and independent of $T$.
    \item The importance sampling weights become unbounded unless treated carefully.
\end{enumerate}
The first change is straightforward. For the second change, we note that if we only estimate the deviation regret up to some $\delta > 0$ accuracy, we can truncate to a finite return of sufficient length, thus controlling the importance sampling weights. See Supplementary Materials~(\ref{sec:inf_consistent_proof}) for details.

For the proof of Theorem~\ref{thm:inf_estimate}, we will define
\begin{itemize}
    \item $\regret_{T,\delta}$, a truncated version of $\regret_T$ that is still $\delta$ close to the deviation regret, for any $\delta > 0$.
    \item $\hat \regret_{T, \delta}^*$, an idealized estimator of the truncated regret that can access future data.
\end{itemize}
We argue that $\hat \regret_{T, \delta}^*$ estimates $\regret_{T, \delta}$ arbitrarily well with high probability. 
However, $\hat \regret_{T, \delta}^*$ uses future data, while a real estimator, $\hat \regret_{T, \delta}$ does not have access to this, resulting in an additional error. Our analysis shows that with increasing data, this error is also driven to be arbitrarily small. Finally, requiring better and better estimates over time by setting $\delta = \delta(T)$ for some $\delta(T)$ that goes to $0$ with time, we arrive at our final estimator $\hat \regret_T = \hat \regret_{T, \delta(T)}$.
Note that $\hat \regret_T$ is not exactly Algorithm~\ref{alg:estimate} with $H = \infty$, as $\hat \regret_T$ may not immediately incorporate a new observed reward into all of its deviation return estimates.

Now that we gave an outline of the proof, we introduce all required notation. All the terms not explicitly introduced here use the definitions provided earlier.
In this section, we denote by $\PP^\agent$ the probability measure on the trajectories induced by the agent-environment interaction, and by $\E^\agent$ the corresponding expectation operator. Similarly, $\PP^\phi$ denotes the probability measure induced by applying the deviation and $\E^\phi$ its expectation operator. Effectively, $\E^\agent$ replaces $\E[ \,\cdot\, | \sigma ]$ and $\E^\phi$ replaces $\E[ \,\cdot\, | \phi(\sigma) ]$. 

We define both infinite and $H$-step returns.
\begin{align*}
    G_t^\infty &= \sum_{i=t}^\infty \gamma^{i-t} R_i, \\
    G_t^{\{H\}} &= \sum_{i=t}^{t+H-1} \gamma^{i-t} R_i.
\end{align*}
With this notation, the deviation regret with infinite returns is
\[
    \rho_T(\pi, \lambda, \env) = \frac1T \sum_{t=1}^T 
        \E^\pi[G_t^\infty | H_{t-1}] - \E^\lambda[G_t^\infty | H_{t-1}],
\]
where $H_{t-1}$ is the random history of the agent. These quantities are bounded for $\gamma < 1$.

\begin{fact} \label{fact:bounded_1}
$|G_t^\infty| \leq \frac{1}{1-\gamma} r^*$
and therefore $|\rho_T| \leq \frac{2}{1-\gamma} r^*$.
Also, as before, $|G_t^{\{H\}}| \leq H r^*$.
\end{fact}

We want to truncate returns, while staying close to the true values. Towards this, let, for any $\delta > 0, \gamma \in [0,1)$, let
\[
    H(\delta, \gamma) = \left\lceil \frac{
        \ln \left( \frac{ r^* }{\delta (1 - \gamma)} \right)
        } {
        1-\gamma
    } \right\rceil,
\]
the \textit{effective horizon}. Choosing $H \geq H(\delta, \gamma)$ guarantees
\[
    | G_t^\infty - G_t^{\{H\}} | < \delta.
\]
Define the truncated regret for any $\delta > 0$ as
\[
    \rho_{T, \delta}(\pi, \lambda, e) = \frac1T \sum_{t=1}^T 
        \E^\pi[G_t^{\{H(\delta, \gamma)\}} | H_{t-1}] - \E^\lambda[G_t^\infty | H_{t-1}].
\]
Then, by construction,
\begin{equation}
    |\rho_T(\pi, \lambda, e) - \rho_{T, \delta}(\pi, \lambda, e)| < \delta.
\end{equation}

To estimate the (truncated) deviation return we use importance sampling, that is
\begin{equation}
    G_t'^{\{H\}} = W_{t:t+H-1} G_t^{\{H\}}
    \quad \text{ with } \quad
    W_{t:i} = \prod_{k=t}^i \frac{\pi(A_k | H_{k-1})}{\lambda(A_k | H_{k-1})} .
\end{equation}
For an agent that takes every action in every step with probability at least $c > 0$, $W_{t:t+H-1} \leq c^{-H}$. With this, we can bound the deviation regret estimate.

\begin{fact} \label{fact:inf_IS_estimate_bounded}
$|G_t'^{\{H\}}| \leq c^{-H}Hr^*$.
\end{fact}

At this point, we can define $\hat \regret_{T, \delta}^*$, the idealized estimator of the truncated regret,
\begin{equation}
    \hat \regret_{T, \delta}^* (\phi, \lambda, e) = \frac1T
        \sum_{t=1}^T {G_t'}^{\{H(\delta, \gamma)\}} - G_t^\infty.
\end{equation}

However, for the practical estimator of regret at timestep $T$ all estimates will naturally be truncated at step $T$, that is
\begin{equation}
    \hat \regret_{T, \delta} (\phi, \lambda, e) = \frac1T
        \sum_{t=1}^T {G_t'}^{\{\min(T-t+1, H(\delta, \gamma))\}} - G_t^{\{T-t+1\}}.
\end{equation}

As stated in the proof outline, we will choose a $\delta(T)$ such that $\delta(T)$ decreases to $0$ in the limit as $T \to \infty$, and use $\hat \regret_{T} = \hat \regret_{T, \delta(T)}$. The particular choice we make is
\[
     \delta(T) = T^{-1/|4\ln(c)|}
\]
With this, we are ready to provide a proof for Theorem~\ref{thm:inf_estimate}.

\begin{proof}
When the deviation $\phi$, agent $\lambda$, and environment $e$ are clear from context, they are omitted from the notation.
We have seen that for any $\delta > 0$, $|\rho_T - \rho_{T, \delta}| < \delta$.
We will show that for any  $\epsilon > 0$,
\begin{equation} \label{eq:approx_inf_dev_reg_est}
    \PP^\agent (|\rho_{T, \delta(T)}
        - \hat \rho_{T, \delta(T)}|  \leq \epsilon)
    \geq 1- f(T, \epsilon),
\end{equation}
for some $f$ with $f(T, \epsilon) \to 0$ as $T \to \infty$. We will refer to the event
\[
    E = \{ |\rho_{T, \delta(T)}
        - \hat \rho_{T, \delta(T)}|  \leq \epsilon \}
\]
as the ``good event''.
For any $\epsilon > 0$, for large enough $T$ such that $\delta(T) < \epsilon$, on the good event, we have
\begin{align*}
    |\rho_T - \hat \rho_{T, \delta(T)}|
    &= |(\rho_T - \rho_{T, \delta(T)})
        + (\rho_{T, \delta(T)} - \hat \rho_{T, \delta(T)})| \\
    &\leq |\rho_T - \rho_{T, \delta(T)}|
        + |\rho_{T, \delta(T)} - \hat \rho_{T, \delta(T)}| \\
    &\leq \epsilon + \epsilon \\
    &= 2\epsilon,
\end{align*}
which in turn shows that for any $\epsilon > 0$
\begin{equation*}
    \lim_{T\to\infty} \PP^\agent (
        |\rho_{T}
        - \hat \rho_{T, \delta(T)}|  \leq \epsilon)
    =1,
\end{equation*}
the statement we set out to prove.

We focus on the estimation problem in Eq.~\ref{eq:approx_inf_dev_reg_est} for the rest of the proof.
As we did in the $H$-step deviation regret estimation case, we argue through the idealized estimator $\hat \rho_{T, \delta(T)}^*$.
We decompose the estimator error as
\begin{align}
    |\regret_{T, \delta(T)} -  \hat \regret_{T, \delta(T)}|
    &= |\regret_{T, \delta(T)} - \hat \regret_{T, \delta(T)}^*
        + \hat \regret_{T, \delta(T)}^* - \hat\regret_{T, \delta(T)}| \nonumber \\
    &\leq   |\regret_{T, \delta(T)} - \hat \regret_{T, \delta(T)}^* |
        +  | \hat \regret_{T, \delta(T)}^* - \hat \regret_{T, \delta(T)} |,   \label{eq:thm2:est_err_decomp}
\end{align}
where we used the triangle inequality. The first term is the estimation error for the idealized estimate, the second term is the difference between the idealized and the truncated estimates. We bound each of these separately.

\paragraph{Idealized Estimator Error.}
We will use the shorthand $H = H(\delta, \gamma)$.
Consider
\begin{align*}
    \Delta_T
    &= T(\hat \regret_{T, \delta(T)}^* - \regret_{T, \delta(T)})  \\
    &= \sum_{t=1}^T ({G_t'}^{\{H\}} - \E^\phi[G_t^{\{H\}} | H_{t-1}])
        + (G_t^\infty - \E^\lambda[G_t^\infty | H_{t-1}]),
\end{align*}
where we regrouped the terms to capture the difference between the true returns and their random estimates.
Let $\delta_t = ({G_t'}^{\{H\}} - \E^\phi[G_t^{\{H\}} | H_{t-1}]) + (G_t^\infty - \E^\lambda[G_t^\infty | H_{t-1}])$, the terms in the sum. We will show that $\delta_t$ is a martingale difference sequence (MDS), hence $\Delta_T$ is a martingale. Towards this, we need $\E^\lambda[\delta_t | H_{t-1}] = 0$ and $|\delta_t|$  bounded. We have
\begin{align*}
    \E^\lambda[\delta_t |  H_{t-1}]
    &= \E^\lambda \Big[ 
        ({G_t'}^{\{H\}} - \E^\phi[G_t^{\{H\}} | H_{t-1}]))
        + (G_t^\infty - \E^\lambda[G_t^\infty | H_{t-1}]))
        \,\Big|\, H_{t-1} \Big] \\
    &= \E^\lambda [ {G_t'}^{\{H\}} |  H_{t-1} ] 
        - \E^\lambda \Big[ \E^\phi[G_t^{\{H\}} | H_{t-1}] \,\Big|\, H_{t-1} \Big]
        + \E^\lambda [ G_t^\infty |  H_{t-1} ]  - \E^\lambda[G_t^\infty | H_{t-1}]
\end{align*}
by linearity of expectation. The last two terms cancel and now we show the first two do as well. Since $\E^\phi[G_t^{\{H\}} | H_{t-1}]$ is the expected deviation return (a constant), the second term itself is the same constant. In the first term $ G_t' =  W_{t:t+H-1} G_t^{\{H\}}$, where $W_{t:t+H-1}$ is the importance sampling that corrects from the agents' behaviour to the deviation's behaviour. That is, it is well known that
\[
    \E^\lambda [ G_t' |  H_{t-1} ]
    = \E^\lambda [ W_{t:t+H-1} G_t^{\{H\}} |  H_{t-1} ]
    = \E^\phi [ G_t^{\{H\}} |  H_{t-1} ],
\]
and we see that the first two terms also cancel. We conclude $\E^\lambda[\delta_t |  H_{t-1}] = 0$.

We bound $|\delta_t|$ by noting that each term in it is bounded. More precisely, we use Facts~\ref{fact:bounded_1} and \ref{fact:inf_IS_estimate_bounded} bounding the individual terms, the triangle inequality, and that for any $X < c$, we also have $\E[X|Y] < c$ for all $Y$.
\begin{align*}
    |\delta_t| 
    &\leq | {G_t'}^{\{H\}} | + | \E^\phi[G_t^{\{H\}} | H_{t-1}])|
        + |G_t^\infty| + |\E^\lambda[G_t^\infty |  H_{t-1}]) | \\
    &\leq c^{-H}Hr^* + Hr^* + \frac{2}{1-\gamma} r^*.
\end{align*}
Therefore, $\Delta_T$ is a martingale with the increments, $\delta_t$, bounded by 
\[
    b(c, H, \gamma, r^*) = (c^{-H}H+H + 2(1-\gamma)^{-1})r^*.
\]
Note that $\Delta_0 = 0$. By the Azuma-Hoeffding inequality, we conclude that for any $\epsilon_1 > 0$ and $T \geq 1$
\begin{equation*}
    \PP^\agent(|\Delta_T| \geq \epsilon_1)
    \leq 2\exp \left( \frac{- \epsilon_1^2}{2 T b(c, H, \gamma, r^*)^2} \right).
\end{equation*}
Since $|\Delta_T| \geq \epsilon_1$ is equivalent to $|\regret_{T, \delta(T)} -  \hat\regret_{T, \delta(T)}^*| \geq \epsilon_1/T$, letting $\epsilon_2 = \epsilon_1 / T$ we can restate the above as
\begin{equation} \label{eq:thm2:hp_bound_partial1}
    \PP^\agent(|\regret_{T, \delta(T)} -  \hat\regret_{T, \delta(T)}^*| \geq \epsilon_2)
    \leq 2\exp \left( \frac{-(\epsilon_2 T)^2}{2T b(c, H, \gamma, r^*)^2} \right)
    = 2\exp \left( \frac{-\epsilon_2^2 T}{2b(c, H, \gamma, r^*)^2} \right).
\end{equation}
Note that $H \in \Theta(\ln(1/\delta))$, not considering its dependence on $r^*, \gamma$.
Furthermore, $\delta(T) = T^{-1/|4\ln(c)|}$. This makes
$H \approx \ln(T^{1/|4 \ln(c)|}) = \frac{\ln T}{4 |\ln(c)|}$, and we have for the denominator in the exponent in Eq.~\ref{eq:thm2:hp_bound_partial1} that 
% Setting $\delta(T) =  \frac{1}{\ln(T^{|4 \ln(c)|^{-1}})} = 
% \frac{4 |\ln(c)|}{\ln T}$,
% $H \approx \ln(T^{|4 \ln(c)|^{-1}}) = \frac{\ln T}{4 |\ln(c)|}$, and we have for the denominator in the exponent in Eq.~\ref{eq:thm2:hp_bound_partial1} that 
\begin{align*}
    2b(c, H, \gamma, r^*)^2
    &= 2 (c^{-H}H+H + 2(1-\gamma)^{-1})^2 {r^*}^2 \\
    &\leq 6  (c^{-2H}H^2 + H^2 + 4(1-\gamma)^{-2}) {r^*}^2  \\
    &\approx 6 \left( (c^{-2\frac{\ln T}{4 |\ln(c)|}} + 1) 
        \left( \frac{\ln T}{4 |\ln(c)|} \right)^2 
        + 4(1-\gamma)^{-2} \right) {r^*}^2.
\end{align*}

Here, using $a^{\log_b(x)} = x^{\log_b(a)}$,
\[
    c^{-2\frac{\ln T}{4 |\ln(c)|}}
    = \exp\left(\frac{\ln T}{2\ln(c)} \cdot \ln c\right)
    = \exp\left(\frac{\ln T}{2}\right) = T^{\frac{1}{2}},
\]
so 
\begin{align*}
    2b(c, H, \gamma, r^*)^2
    &\leq  C_0 \left( (T^{\frac{1}{2}} + 1) 
        \left( \frac{\ln T}{4 |\ln(c)|} \right)^2 
        + 4(1-\gamma)^{-2} \right) {r^*}^2 \\
    &\leq  C_1 \left(  
        T^{\frac{3}{4}} |2\ln(c)|^{-2}
        + 4(1-\gamma)^{-2} \right) {r^*}^2 \\
    &\leq  C_1 T^{3/4} 
        ( |2\ln(c)|^{-2} + 4(1-\gamma)^{-2} ) {r^*}^2,
\end{align*}
for some constants $C_i \in \mathbb{R}$ and for sufficiently large $T$. We used that $\ln T \leq C_2 T^{1/8}$ for large enough $T$ and some $C_2 > 0$, and that $T^{3/4} \geq 1$.
Plugging this back into Eq.~\ref{eq:thm2:hp_bound_partial1} we find
\begin{equation} \label{eq:thm2:hp_bound_partial2}
    \PP^\agent(|\regret_{T, \delta(T)} -  \hat\regret_{T, \delta(T)}^*| \geq \epsilon_2)
    % \leq 2\exp \left( \frac{-\epsilon_2^2 T}{2b(c, H, \gamma, r^*)^2} \right)
    \leq 2\exp \left( \frac{-\epsilon_2^2 T^{1/4}}{
        C_1 \left( |2\ln(c)|^{-2}
            + 4(1-\gamma)^{-2} \right) {r^*}^2 } \right).
\end{equation}
We see that the idealized estimate gets close to the true deviation with increased interaction time $T$. This completes the bound on the first error term.

\paragraph{Error due to truncation.}
We turn our attention to the second term of Eq.~\ref{eq:thm2:est_err_decomp},  $|\hat \regret_{T, \delta(T)}^* - \hat \regret_{T, \delta(T)}|$. The analysis will use $\delta$ for $\delta(T)$ and $H$ for $H(\delta, \gamma)$ except when the dependence on the arguments is important.
By definition and some algebra,
\begin{align*}
    T|\hat \regret_{T, \delta}^* - \hat \regret_{T, \delta}|
     &= \left|\sum_{t=1}^T  \left({G_t'}^{\{H\}} - G_t^\infty \right)
        -  \left( {G_t'}^{\{\min(T-t+1, H)\}} - G_t^{\{T-t+1\}} \right) \right| \\
     &= \left|\sum_{t=1}^T  \left({G_t'}^{\{H\}} - {G_t'}^{\{\min(T-t+1, H)\}} \right)
        -  \left(G_t^\infty  - G_t^{\{T-t+1\}} \right) \right| \\
     &= \left|\sum_{t=1}^T  \left({G_t'}^{\{H\}} - {G_t'}^{\{\min(T-t+1, H)\}} \right)
        -  \gamma^{T-t+1} G_{T+1}^\infty \right| \\
     % &= |T(\tilde G_{T,\delta}' - \hat G_{T,\delta}')| \\
     % &= |\sum_{t=1}^T {G_t'}^{[\min(T-t+1, H)]} - {G'}_t^{[H]}| \\
     % &= |\sum_{t=1}^T \sum_{i=t}^{\min(T, t+H-1)} \gamma^{i-t} W_{t:i} R_i
     %            - \sum_{i=t}^{t+H-1}\gamma^{i-t} W_{t:i} R_i|.
     &\leq \left|\sum_{t=1}^T  \left({G_t'}^{\{H\}} - {G_t'}^{\{\min(T-t+1, H)\}} \right) \right|
        + \left| \sum_{t=1}^T   \gamma^{T-t+1} G_{T+1}^\infty \right|,
\end{align*}
where in the last step we used the triangle inequality.
The second term is bounded as
\[
    \left| \sum_{t=1}^T   \gamma^{T-t+1} G_{T+1}^\infty \right|
    \leq \frac{r^*}{1-\gamma} \sum_{t=1}^T  \gamma^{T-t+1} 
    \leq \frac{r^*}{(1-\gamma)^2}.
\]
For the first term, when $H \leq T - t +1$, the difference is $0$.
We continue with the other case, that is, $\min(T-t+1, H) = T-t+1$.
% For the next calculation the notation $G_t^{[T]}$ from the appendix, which truncates returns so that the last included step is $T$, will be useful.
{\allowdisplaybreaks
\begin{align*}
    \left|\sum_{t=1}^T {G_t'}^{\{H\}} - {G_t'}^{\{T-t+1\}} \right|
     &= \left|\sum_{t=1}^T  \sum_{i=t}^{t+H-1} \gamma^{i-t} W_{t:i} R_i 
        - \sum_{i=t}^{T} \gamma^{i-t} W_{t:i} R_i \right| \\
     &= \left|\sum_{t=1}^T  \sum_{i=T+1}^{t+H-1} \gamma^{i-t} W_{t:i} R_i  \right| \\
     % &= \left|\sum_{t=1}^T {G_t'}^{\{H\}} - {G_t'}^{[T]} \right| \\
     % &= \left|\sum_{t=1}^T \gamma^{T-t+1} {G_{T+1}'}^{\{H-(T-t+1)\}} \right| \\
     % &= |\sum_{t=1}^T \gamma^{T-t+1}  \sum_{i=T+1}^{t+H-1} \gamma^{i-t} W_{t:i} R_i| \\
     &\leq \sum_{t=1}^T \sum_{i=T+1}^{t+H-1} \gamma^{i-t} W_{t:i} \, r^* \\
     &\leq \sum_{t=1}^T \sum_{i=T+1}^{t+H-1} \gamma^{i-t} \left( \prod_{k=t}^i \frac1c \right)  \, r^* \\
     &= \sum_{t=T-H+1}^T \sum_{i=T+1}^{t+H-1} \gamma^{i-t} c^{-(i-t+1)} \, r^* \\
     &\leq c^{-1} \sum_{t=T-H+1}^T \sum_{i=T+1}^{T+H-1} \gamma^{i-t} c^{-(i-t)} \, r^* \\
     &\leq c^{-1} \sum_{t=T-H+1}^T \sum_{h=0}^{2H} (\gamma / c)^h \, r^* \\
     &= r^* c^{-1} H \sum_{h=0}^{2H} (\gamma / c)^h.
\end{align*}}
We control $f(H, \gamma/c) := \sum_{h=0}^{2H} (\gamma / c)^h$ dependent on where $\gamma / c$ lands compared to 1.
\begin{itemize}
    \item If $\gamma / c < 1$,  $ f(H, \gamma/c) \leq (1-\gamma/c)^{-1}$.
    \item If $\gamma / c = 1$,  $ f(H, \gamma/c) \leq 2H + 1$.
    \item If $\gamma / c > 1$,  $ f(H, \gamma/c) \leq \frac{(\gamma/c)^{2H+1} - 1}{\gamma/c - 1}$.
\end{itemize}
In conclusion, for any $H, \delta, T$
\[
    T|\hat \regret_{T, \delta}^* - \hat \regret_{T, \delta}|
    \leq \frac{r^*}{(1-\gamma)^2} + r^* c^{-1} H f(H, \gamma/c).
\]
We now make explicit the dependency of $H$ on $T$ through $\delta(T)$, while continuing to suppress the dependency of $H$ on $\gamma$ and $r^*$.
As before, for $\delta(T) = T^{-1/|4\ln(c)|}$
we have $H \approx \ln(T^{1/|4 \ln(c)|}) = \frac{\ln T}{4 |\ln(c)|}$
and
\[
    T|\hat \regret_{T, \delta(T)}^* - \hat \regret_{T, \delta(T)}|
    \lesssim \frac{r^*}{(1-\gamma)^2} + r^* c^{-1} \frac{\ln T}{4 |\ln(c)|} \, f\left (\frac{\ln T}{4 |\ln(c)|}, \gamma/c \right),
\]
where $f$ scales the worst in $H$ for the $\gamma / c > 1$ case. In this setting,
\[
    f(H, \gamma/c)
    \leq \frac{(\gamma/c)^{2H+1} - 1}{\gamma/c - 1}
    \leq  \frac{(\gamma/c)\, T^{\frac{\ln(\gamma/c)}{2|\ln(c)|}}-1}{\gamma/c-1},
\]
where, noting the range of $c$ and $\gamma$, we see
$ \frac{\ln(\gamma/c)}{2|\ln(c)|} = \frac{-\ln(c)}{2(-\ln(c))} + \frac{\ln(\gamma)}{-2\ln(c)} \leq 1/2$, since the second term is negative. That is, we are guaranteed that $|\hat \regret_{T, \delta(T)}^* - \hat \regret_{T, \delta(T)}| \in O(T^{-1/2})$. With this, we are ready to finish up the proof.

\paragraph{The Estimation Problem.}
We originally set out to analyze the truncated regret estimation problem introduced in Eq.~\ref{eq:approx_inf_dev_reg_est}, through the error decomposition of Eq.~\ref{eq:thm2:est_err_decomp}. We now provided a bound for each of the error terms and we are ready to combine them for the desired result. Eq.~\ref{eq:thm2:est_err_decomp} stated that
\begin{align*}
    |\regret_{T, \delta(T)} -  \hat \regret_{T, \delta(T)}|
    &\leq   |\regret_{T, \delta(T)} - \hat \regret_{T, \delta(T)}^* |
        +  | \hat \regret_{T, \delta(T)}^* - \hat \regret_{T, \delta(T)} |.
\end{align*}
We chose $\delta(T) = T^{-1/|4\ln(c)|}$, which indeed approaches $0$ as $T \to \infty$.
Then, we saw that for all $\epsilon_2 > 0$ there exist constants $C_0, T_0$ (independent of $\epsilon_2$) such that for all $T > T_0$
\[
    \PP^\agent(|\regret_{T, \delta(T)} -  \hat\regret_{T, \delta(T)}^*| \geq \epsilon_2)
    \leq 2\exp \left( \frac{-\epsilon_2^2 T^{1/4}}{
        C_0 \left( |2\ln(c)|^{-2}
            + 4(1-\gamma)^{-2} \right) {r^*}^2 } \right).
\]
Finally, we saw that there exists $f_0$ and $T_1$ such that for any $T \geq T_1$ 
\[
    |\hat \regret_{T, \delta(T)}^* - \hat \regret_{T, \delta(T)}|
    \leq \frac{1}{\sqrt{T}} f_0 (r^*, \gamma, c).
\]
We can conclude that
\begin{align*}
    |\regret_{T, \delta(T)} -  \hat \regret_{T, \delta(T)}|
    &\leq   |\regret_{T, \delta(T)} - \hat \regret_{T, \delta(T)}^* |
        + \frac{1}{\sqrt{T}} f_0 (r^*, \gamma, c).
\end{align*}
Then, 
$\forall \varepsilon > 0$, choosing $T \geq \max(T_0, T_1, ( f_0(r^*, \gamma, c) / \varepsilon)^2)$, we have $\varepsilon \geq f_0(r^*, \gamma, c) / \sqrt{T}$, and
\begin{align*}
    \PP^\agent \left( |\regret_{T, \delta(T)} -  \hat \regret_{T, \delta(T)}| \geq 2\varepsilon  \right)
    & \leq
    \PP^\agent \left( |\regret_{T, \delta(T)} - \hat \regret_{T, \delta(T)}^* |
        +  f_0(r^*, \gamma, c) / \sqrt{T} \geq 2\varepsilon  \right) \\
    & \leq
    \PP^\agent \left( |\regret_{T, \delta(T)} - \hat \regret_{T, \delta(T)}^* | 
            \geq \varepsilon  \right) \\
    & \leq 2\exp \left( \frac{-\varepsilon^2 T^{1/4}}{
        C_0 \left( |2\ln(c)|^{-2}
            + 4(1-\gamma)^{-2} \right) {r^*}^2 } \right).
\end{align*}
This bound indeed goes to $0$ as $T \to \infty$, so the proof is complete.
\end{proof}

\section{Experimental Details}
We used the default hyperparameters for PPO that are commonly used for MuJoCo environments based on ~\citep{shengyi2022the37implementation}. We show those hyperparameters in Table~\ref{tab:ppo_hypers}.

\begin{table}[htb]
    \centering
    \begin{tabular}{l l}
        \hline
        Name & Default Value \\
        \hline
        Policy Network & (64, tanh, 64, tanh, Linear) + Standard deviation variable \\
        Value Network & (64, tanh, 64, tanh, Linear) \\
        Rollout Length & 2048 \\
        Epochs & 4 \\
        Mini-batch size & 64 \\
        GAE, $\lambda$ & 0.95 \\
        Discount factor, $\gamma$ & 0.99 \\
        Clip parameter & 0.2 \\
        Input Normalization & True \\
        Advantage Normalization & True \\
        Value function loss clipping & True \\
        Max Gradient Norm & 0.5 \\
        Optimizer & Adam \\
        Actor step size & 0.0003 \\
        Critic step size & 0.0003 \\
        Optimizer $\epsilon$ & $1 \times 10^{-5}$ \\
        \hline
    \end{tabular}
    \caption{The hyperparameters used for PPO in the continuing Swimmer experiment.}
    ~\label{tab:ppo_hypers}
\end{table}

% Content that appears after the references are not part of the ``main text,'' have no page limits, are not necessarily reviewed, and should not contain any claims or material central to the paper. 
% %
% If your paper includes supplementary materials, use the \begin{center}
%     {\tt {\textbackslash}beginSupplementaryMaterials} 
% \end{center}
% command as in this example, which produces the title and disclaimer above. 
% %
% If your paper does not include supplementary materials, this command can be removed or commented out.

\end{document}

%% file: macros.tex
\def\Action{{\cal A}}
\def\Obs{{\cal O}}
\def\Hist{{\cal H}}
\def\State{{\cal S}}
\def\env{e}
\def\reward{R}
\def\agent{\lambda}
\def\Agent{\Lambda}
\def\AgentBasis{\Lambda_b}
\def\regret{\rho}
\def\dev{\phi}
\def\Dev{\Phi}

\newcommand{\PP}{\mathbb{P}}
\newcommand{\E}{\mathbb{E}}